\documentclass{article}
\usepackage{amssymb}
\usepackage{amsmath}
\usepackage{graphicx}
\usepackage{hyperref}
\usepackage{wrapfig}
\usepackage{booktabs}
\usepackage{soul}
\usepackage{color}
\usepackage{amsthm}
\newtheorem{theorem}{Theorem}
\usepackage{subcaption}

\usepackage[preprint]{corl_2025}

\title{Latent Theory of Mind: A Decentralized Diffusion Architecture for Cooperative Manipulation}

\author{
  Chengyang He\\
  National University of Singapore\\
  Singapore, Republic of Singapore \\
  \texttt{hecy@stanford.edu} \\
  \And
  Gadiel Sznaier Camps \\
  Stanford University \\
  California, United States \\
  \texttt{gsznaier@stanford.edu} \\
  \And
  Xu Liu \\
  Stanford University \\
  California, United States \\
  \texttt{liuxujsw@stanford.edu} \\
  \And
  Mac Schwager \\
  Stanford University \\
  California, United States \\
  \texttt{schwager@stanford.edu} \\
  \And
  Guillaume Sartoretti \\
  National University of Singapore \\
  Singapore, Republic of Singapore \\
  \texttt{guillaume.sartoretti@nus.edu.sg} \\
}

\begin{document}
\maketitle


\begin{abstract}
    We present Latent Theory of Mind (LatentToM), a decentralized diffusion policy architecture for collaborative robot manipulation.
    Our policy allows multiple manipulators with their own perception and computation to collaborate with each other towards a common task goal with or without explicit communication.
    Our key innovation lies in allowing each agent to maintain two latent representations: an \emph{ego} embedding specific to the robot, and a \emph{consensus} embedding trained to be common to both robots, despite their different sensor streams and poses. 
    We further let each robot train a decoder to infer the other robot's ego embedding from their consensus embedding, akin to ``theory of mind'' in latent space.
    Training occurs centrally, with all the policies' consensus encoders supervised by a loss inspired by sheaf theory, a mathematical theory for clustering data on a topological manifold. 
    Specifically, we introduce a first-order cohomology loss to enforce sheaf-consistent alignment of the consensus embeddings.
    To preserve the expressiveness of the consensus embedding, we further propose structural constraints based on theory of mind and a directional consensus mechanism.
    Execution can be fully distributed, requiring no explicit communication between policies.
    In which case, the information is exchanged implicitly through each robot's sensor stream by observing the actions of the other robots and their effects on the scene.
    Alternatively, execution can leverage direct communication to share the robots' consensus embeddings, where the embeddings are shared once during each inference step and are aligned using the sheaf Laplacian. 
    While we tested our method using two manipulators, our approach can naturally be extended to an arbitrary number of agents. 
    In our hardware experiments, LatentToM outperforms a naive decentralized diffusion baseline, and shows comparable performance with a state-of-the-art centralized diffusion policy for bi-manual manipulation.
    Additionally, we show that LatentToM is naturally robust to temporary robot failure or delays, while a centralized policy may fail.
    More information can be found in \url{https://stanfordmsl.github.io/LatentToM/}.
\end{abstract}



\section{Introduction}
\label{sec:intro}

Robotic arm manipulation refers to the process by which robotic arms perceive, grasp, move, rotate, or otherwise interact with objects, typically to accomplish precise or complex tasks~\cite{billard2019trends}. 
This technology plays a critical role in a wide range of applications, including industrial automation~\cite{domel2017toward}, warehousing, logistics~\cite{benali2018dual}, medical procedures~\cite{ginoya2021historical}, and agriculture scouting~\cite{guri2024hefty}. 
The growing interest in humanoid robots~\cite{ze2024generalizable,figure2024helix} has further highlighted the importance of advanced manipulation capabilities, especially those involving dexterous and coordinated arm movements.

Recently, the Diffusion Policy~\cite{chi2023diffusion} has attracted the attention of the community as one of the state-of-the-art robotic arm manipulation solutions. 
It excels in generating smooth, multi-modal, and long-horizon trajectories, outperforming many existing methods~\cite{shafiullah2022behavior,mandlekar2021matters,florence2022implicit,gupta2019relay}. 
Since it inherits the strong ability of the diffusion models~\cite{ho2020denoising} to process high-dimensional data, it can achieve decent performance even when dealing with long-horizon dual-arm cooperative tasks.
While centralized frameworks are effective for current dual-arm applications, the lack of multi-agent training data and their inherent fragility, such as poor scalability, difficulty in training, and sensitivity to failure, has motivated the community to explore decentralized alternatives. 
In multi-arm~\cite{ha2020learning} or multi-agent systems~\cite{sartoretti2019primal,he2024social}, a decentralized framework improves system robustness to outside disturbances while providing more flexibility. 
However, achieving coherent and cooperative behavior in a decentralized manner is non-trivial. 
Each robotic arm operates with partial observations and may be subject to domain shifts, making it difficult to maintain global consistency across the robotic arms. Moreover, even a slight inconsistency between agents can easily lead to task failure, especially when performing delicate cooperative tasks. 
Therefore, a key challenge is to design a consensus representation that captures overlapping information across all agents, despite the incompleteness of their local observations. 
Specifically, in dual-arm scenarios, we need to stably and efficiently train two decentralized policies that align partial information while maintaining independence.

To address this challenge, we propose LatentToM, a decentralized diffusion policy, that enables each robotic arm to independently generate motion trajectories while maintaining coordination with others through a shared consensus representation.
We propose a structured separation of observations, encoding and maintaining them independently to form an ego embedding and a consensus embedding for each arm.
By integrating insights from sheaf theory, we further impose consistency constraints on the consensus embedding derived from the observations of each robotic arm. 
Specifically, by minimizing the sheaf 1-cohomology, the two arms are encouraged to produce globally consistent interpretations from the consensus embedding. 
This ensures that both arms develop a unified understanding of key shared states, which is important for decentralized collaborative decision-making.
To achieve this sheaf-theoretic consistency during training, we introduce a sheaf consistency loss as an auxiliary objective~\cite{liao2025sigma}, penalizing discrepancies in the implicit representations of overlapping observations. 
The objective is to reduce coordination errors caused by inconsistent interpretations of the same scene. 
For example, when one arm makes an action, the other arm can synchronously understand this change and make reactive decisions to ensure the completion of the task.
To further ensure that the resulting consensus embedding is not only numerically consistent but also expressive, we incorporate two additional constraints:
\begin{itemize}
    \item Theory of Mind (ToM)-inspired constraint: Each agent is trained to use its own consensus embedding to infer the other’s ego embedding, encouraging the representation to retain rich and discriminative information and preventing it from collapsing into trivial solutions.
    \item Directional consensus mechanism: We guide the lower-confidence embedding to align with the higher-confidence one, thus reducing the risk caused by uncertain representations and maintaining information richness. 
\end{itemize}
For more stable performance, we use a classic bidirectional consistency synchronization operator as sheaf Laplacian. 
This operator adjusts the consensus embeddings of multiple arms via online post-processing to promote consensus. 
The trade-off is that this approach requires the arms to be capable of communication and to perform one round of information exchange before every decision.
Through comprehensive comparisons with the vanilla diffusion policy and its naive decentralized variant, we demonstrate the effectiveness of our approach in achieving consistent, expressive, and collaborative behavior in cooperative multi-arm tasks.
In addition, we integrated our code into the vanilla diffusion policy codebase to provide a training dataset, making it easy for anyone interested in our approach to use and deploy directly.


\section{Related Work}
\label{sec:relatedwork}

\begin{wrapfigure}{r}{0.4\textwidth}
  \centering
  \includegraphics[width=0.4\textwidth]{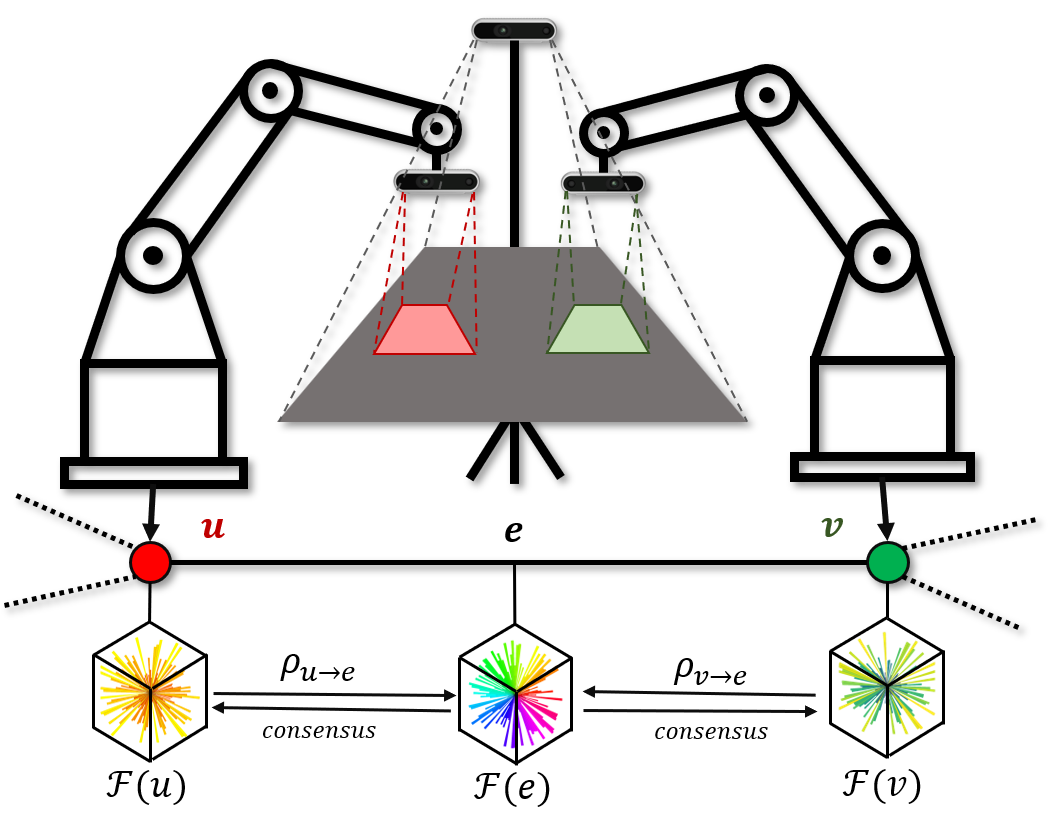}
  \caption{Multi-arm robotic system. In our setup, the system consists of two robotic arms, each equipped with an end-effector camera, represented by the red and green areas indicating their respective fields of view. 
  Additionally, a third-person camera observes the overlapping workspace between the two arms, shown in gray. The bottom part illustrates the consensus embeddings generated using sheaf theory from our collected data.}
  \label{fig:overview}
  \vspace{-12pt}
\end{wrapfigure}

Building upon the success of diffusion models in high-dimensional domains like image and audio generation, diffusion policies have emerged as one of the most advanced visuomotor policies for robotic arm manipulation~\cite{chi2023diffusion}. By modeling robot motion as a conditional denoising diffusion process, they can generate multi-modal, smooth, and long-horizon trajectories for a wide range of tasks. In particular, since they learn from expert demonstrations, they are especially useful for solving complex problems where data collection is costly or difficult to quantify through hand-engineered rewards \cite{zhu2023diffusion,chi2023diffusion}. Several variants, such as 3D diffusion~\cite{ze20243d} and equivariant diffusion policy~\cite{wang2024equivariant,yang2024equibot}, further enhance their data efficiency and generalization ability by leveraging improved input structures and equivariant representation learning, respectively. Others improve their ability to solve long-horizon or multi-task objectives by applying learning techniques \cite{mishra2023generative, ren2024diffusion, wang2024sparse}, using methods such as high-level task planners \cite{ma2024hierarchical}, subproblem decomposition \cite{razmjoo2025ccdp} or trajectory guidance \cite{fan2025diffusion}. Likewise, due to their effectiveness in manipulation tasks, there has also been a push towards improving their precision \cite{wu2024tacdiffusion} and their ability to adapt based on environmental factors \cite{rana2023sayplan, ke20243d}.

Typically, these approaches have focused on controlling up to two agents through a centralized policy. The policy is trained from data collected from a single user controlling both agents \cite{fu2024mobile, chi2024universal} and is usually provided a fused multi-view and action representation \cite{chi2023diffusion,chi2024universal}. 
However, reliance on a centralized framework introduces limitations: training becomes increasingly difficult and system reliability will deteriorate as the scale of the system grows. This challenge is further complicated by the scarcity of multi-arm datasets that include explicit, necessary, and compelling examples of collaboration, as well as the difficulty of collecting such data. This has led to a widely recognized shift in the multi-agent learning community, from centralized to decentralized methods~\cite{oroojlooy2023review,lian2017can}. 
A well-known example is multi-agent planning, where, despite the availability of well-established centralized solvers~\cite{shome2018fast,van2005roadmap,li2022mapf}, research increasingly favors decentralized frameworks due to their superior robustness and scalability in real-world deployments~\cite{ha2020learning,sartoretti2019primal}. 
Recent studies also achieved complex coordination among multi-agent behaviors through diffusion models based on attention mechanisms~\cite{zhu2024madiff}.
A core challenge in decentralized frameworks lies in enabling effective collaboration among individual agents. 
Specifically, agents should reach a consensus, either on global task goals or local coordination states, to execute cooperative behaviors effectively. 
Sheaf theory~\cite{robinson2017sheaves,curry2014sheaves} offers a principled mathematical framework to address this challenge, by providing tools for integrating locally distributed information into a globally consistent structure~\cite{robinson2013understanding}. 
It aligns naturally with decentralized learning paradigms, where agents operate on partial observations or interact with neighbors to coordinate their behaviors~\cite{bodnar2022neural}. 
In our context, we leverage sheaf theory to formalize the global task representation and extract consistent intermediate embeddings across agents, enabling decentralized policies to maintain coherence among local policies and complete complex cooperative tasks.


\section{Method}
\label{sec:ddp}

\subsection{Model Training}
\label{subsec:mt}

The architecture of our decentralized diffusion policy, LatentToM, is designed to accommodate systems comprising of $N$ independent robotic arms, where each arm maintains its own neural network. 
In this section, we illustrate the detailed workings of our LatentToM architecture using a two-arm robotic system as an example. 
As depicted in Fig~\ref{fig:overview}, each robotic arm is equipped with local sensing capabilities, including an end-effector-mounted camera and pose sensors, and shares global scene information captured by a fixed third-person view camera. 
We define the dual-arm system as a graph $\mathcal{G}=(\mathcal{V},\mathcal{E})$, where each robotic arm is represented as a node $u,v\in\mathcal{V}$, and an edge exists between two nodes if their respective workspaces and tasks overlap. 
Each node $u/v$ has an observation space $o_{u/v}$, which we explicitly divide into two subspaces: $o_u=[o_u^{ego},o_u^{con}]$.
Here, $o_u^{ego}$ consists of the arm’s end-effector image and pose, while $o_u^{con}$ corresponds to the third-person view.
We employ two visual encoders~\cite{chi2023diffusion}, $\phi_u^{ego}$ and $\phi_u^{con}$, to process the private and shared observations separately.
This results in a complete embedding $h_u=[h_u^{ego},h_u^{con}]$, where $h_u^{ego}=\phi_u^{ego}(o_u^{ego})$ and $h_u^{con}=\phi_u^{con}(o_u^{con})$.

Next, we formalize the dual-arm system as a cellular sheaf\footnote{For a more in detailed explanation on Sheaf theory, we refer the readers to \cite{bredon2012sheaf}.} $\mathcal{F}$ defined over the graph $\mathcal{G}$.
Specifically, the sheaf $\mathcal{F}$ assigns a vector space to each node $u\in\mathcal{V}$ as: $\mathcal{F}(u)={h_u^{con}}\subseteq\mathbb{R}^{d_{con}}$. 
For each edge $e=(u,v)\in\mathcal{E}$, we define the corresponding restriction map as:
\begin{equation}
    \begin{aligned}
        \rho_{u\rightarrow e}:\mathcal{F}(u)\rightarrow\mathcal{F}(e),~\rho_{v\rightarrow e}:\mathcal{F}(v)\rightarrow\mathcal{F}(e)
    \end{aligned}
\end{equation}
where $\mathcal{F}(e)$ denotes the overlapping subspace associated with edge $e$.
In our context, we assume that the overlapping subspace is aligned with the consensus embedding space, i.e., $\mathcal{F}(e)=\mathbb{R}^{d_{con}}$. 
To achieve numerical consistency between nodes in the consensus embedding space, we minimize a loss function derived from the first-order cohomology defined in sheaf theory, which is:
\begin{equation}\label{eq:mse}
    \begin{aligned}
        \mathcal{L}_{nc}=\sum_{e=(u,v)\in\mathcal{E}}||\rho_{u\rightarrow e}(h_u^{con})-\rho_{v\rightarrow e}(h_v^{con})||_2^2.
    \end{aligned}
\end{equation}
According to sheaf theory, this loss function measures the first-order cohomology of the sheaf $\mathcal{F}$. 
Minimizing $\mathcal{L}_{nc}$ encourages the system to approach a more coherent state (i.e., approaching a global section) in which the consensus embeddings across all nodes achieve global consistency in the numerical perspective, thereby naturally yielding a form of consensus learning.

Although Equation \ref{eq:mse} enforces strict numerical consistency between nodes (i.e., arms in our context), this constraint alone may lead to representation collapse, where the consensus embedding $h_u^{con}$ (i.e., consensus), degenerates into a constant or overly simplified vector, thus limiting its expressiveness and information capacity.
To preserve the richness and expressiveness of the consensus embedding $h_u^{con}$, we draw inspiration from the concept of ToM, which is the ability of an agent $u$ to internally reason about the states, intentions, and goals of others $v$.
Specifically, we require each arm to use its own consensus embedding $h_u^{con}$ to infer the ego embedding of the other arm $h_v^{ego}$. 
This encourages the consensus embedding to carry sufficient information about the global context, thus preventing trivial or collapsed solutions.
Building upon Equation \ref{eq:mse}, we introduce an additional loss function $\mathcal{L}_{tom}$:
\begin{equation}\label{eq:tom}
    \begin{aligned}
        \mathcal{L}_{tom} = \sum_{(u,v)\in\mathcal{E}}||h_v^{ego} - \psi_{u\rightarrow v}(h_u^{con},h_v^{ego})||_2^2 + (v\leftrightarrow u)
    \end{aligned}
\end{equation}
where $h_v^{ego}$ is the ego embedding of arm $v$. 
$\psi_{u\rightarrow v}:\mathbb{R}^{d_{con}}\rightarrow\mathbb{R}^{d_{ego}}$ defines a cross-agent prediction function (i.e., ToM predictor)\footnote{Implementation details about ToM predictor can be found in Appendix \ref{appx:tom}.}, which is used by node $u$ to predict the ego embedding $\hat{h}^{ego}_{v|u}$ of node $v$.
In the context of sheaf theory, this serves as an explicit structural constraint on the restriction maps $\rho_{u\rightarrow e}$. 
In other words, the consensus embedding at each node must not only satisfy the global consistency required for consensus but also retain enough semantic information to be meaningfully mapped into the ego feature space $\mathcal{F}(v)$ of neighboring nodes, thereby ensuring that global information is preserved in consensus $\mathcal{F}(e)$.

Although Equations \ref{eq:mse} and \ref{eq:tom} constrain the alignment and information prediction consistency between consensus embeddings $h_u^{con}$ and $h_v^{con}$, ensuring that the learned consensus retains sufficient expressiveness, they implicitly assume that the two embeddings have equal representational quality. 
However, in real-world scenarios, certain arms may have higher-quality observations due to factors such as a better field of view, more stable motion, or more task-relevant sensory inputs.
This raises an important question: who should guide the alignment when the quality of consensus embeddings is unequal?
To address this, we introduce a directional consensus mechanism, in which each arm learns a confidence score $c_{u/v}\in [0,1]$ to indicate the reliability of its consensus embedding.\footnote{Implementation details about confidence predictor can be found in Appendix \ref{appx:conf}.} 
This enables a directionally asymmetric alignment, also referred to as one-way consistency, in which the lower-confidence embedding is guided toward the higher-confidence one.
To implement this mechanism, we define a confidence loss $\mathcal{L}_{conf}$ as follows:
\begin{equation}\label{eq:conf}
    \begin{aligned}
        \mathcal{L}_{conf}=&\underbrace{(1+|\Delta c|)}_{\text{difference-weighting}}\cdot [\mathbf{1}_{c_u\geq c_v}\cdot ||h_v^{con} - h_u^{con}||+\mathbf{1}_{c_v\geq c_u}\cdot ||h_u^{con} - h_v^{con}||]\\
        &+ \lambda_{ent}\cdot [\mathcal{H}(c_u)+\mathcal{H}(c_v)],
    \end{aligned}
\end{equation}
with $\Delta c=c_u-c_v$ and $\mathcal{H}(c)=-c\log c-(1-c)\log{(1-c)}$.
The use of the difference-weighting term in Equation \ref{eq:conf} aims to dynamically adjust the penalty on embedding alignment errors based on the confidence difference between two nodes $u$ and $v$. 
In other words, if the confidence scores of the two agents are similar, indicating that they are ``mutually reliable", there is little need to apply a strong penalty during alignment. 
However, if there is a significant confidence gap (e.g., one node is highly reliable while the other is not), then the alignment error should be penalized more heavily, encouraging the less trustworthy node to more closely align with the more reliable one.
In addition, we introduce a confidence entropy term $\lambda_{ent}\cdot [\mathcal{H}(c_u)+\mathcal{H}(c_v)]$ as a soft regularization mechanism to prevent the confidence scores from collapsing into extreme values. 
This not only stabilizes the optimization process during the early stages of training but also enhances the model’s ability to distinguish between reliable and unreliable embeddings over time.
Since the embedding distribution tends to be noisy and unstable at the beginning of training, directly applying one-way consistency based on early confidence differences may lead to unreliable guidance.
By incorporating this entropy term, the initial confidence values are naturally encouraged to remain near 0.5, effectively avoiding premature overconfidence and allowing the model to learn a more robust confidence estimation as training progresses.
In summary, the total auxiliary loss used during training is given by $\mathcal{L}_{tot}=\alpha\mathcal{L}_{nc}+\beta\mathcal{L}_{tom}+\gamma\mathcal{L}_{conf}$, where $\alpha$, $\beta$, and $\gamma$ are hyperparameters that must be carefully tuned to suit different tasks.

\subsection{Model Inference}
The training paradigm described in the previous subsection yields a fully decentralized model that operates without any explicit information exchange or communication. 
Each node is able to infer a consensus and coordinate its actions using only its own partial observations.
However, during inference, the diffusion policy model requires a relatively long rollout time to generate and execute an action sequence based on current observations. 
In a decentralized setting, this can lead to problematic behaviors such as mutual avoidance, excessive waiting, or repeated local actions, ultimately resulting in livelocks.
For instance, one arm may take an action that is misinterpreted by the other, leading to an unnecessary reaction. 
This can trigger a series of back-and-forth adjustments between the arms, without ever reaching a clear resolution, resulting in an oscillating or indecisive policy.
This issue arises because, unlike the perfectly synchronized setting during training, we cannot guarantee that the consensus embeddings of the two arms remain fully synchronized during inference. 

To mitigate such decision-making deviations caused by local inconsistencies, we introduce the sheaf Laplacian~\cite{hansen2019learning,wei2021persistent} as an online adjustment mechanism during inference. 
This method provides a lightweight, model-agnostic ``consensus repair" process that does not require modifying the trained model. 
By iteratively updating the consensus embeddings of the two nodes, the sheaf Laplacian gradually brings them closer together, promoting consistency and behavioral stability. 
Specifically, we use a classic bidirectional consistency synchronization operator:
\begin{equation}\label{eq:sl}
    \begin{aligned}
        \begin{bmatrix}
        h_{u,t+1}^{con} \\ h_{v,t+1}^{con}
        \end{bmatrix} = 
        \underbrace{
        \begin{bmatrix}
        1 - \eta & \eta \\
        \eta & 1 - \eta
        \end{bmatrix}}_{\text{consistency operator}}
        \cdot
        \begin{bmatrix}
        h_{u,t}^{con} \\ h_{v,t}^{con}
        \end{bmatrix},
    \end{aligned}
\end{equation}
which is equivalent to performing a low-order sheaf Laplacian step on the two embeddings\footnote{Proof can be found at Appendix \ref{appx:proof}.}.
The final adjusted embeddings $h_{u,T}^{con}$, $h_{v,T}^{con}$ are then used for downstream decision-making.
It is worth noting that if the sheaf Laplacian method described in Equation \ref{eq:sl} is used for online consensus adjustment, a one-step information exchange (communication) between the two nodes is required.


\section{Experimental Results}
\label{sec:result}

\subsection{Experiment Setting}
\begin{figure}[t]
\centering
\includegraphics[width=4.3in]{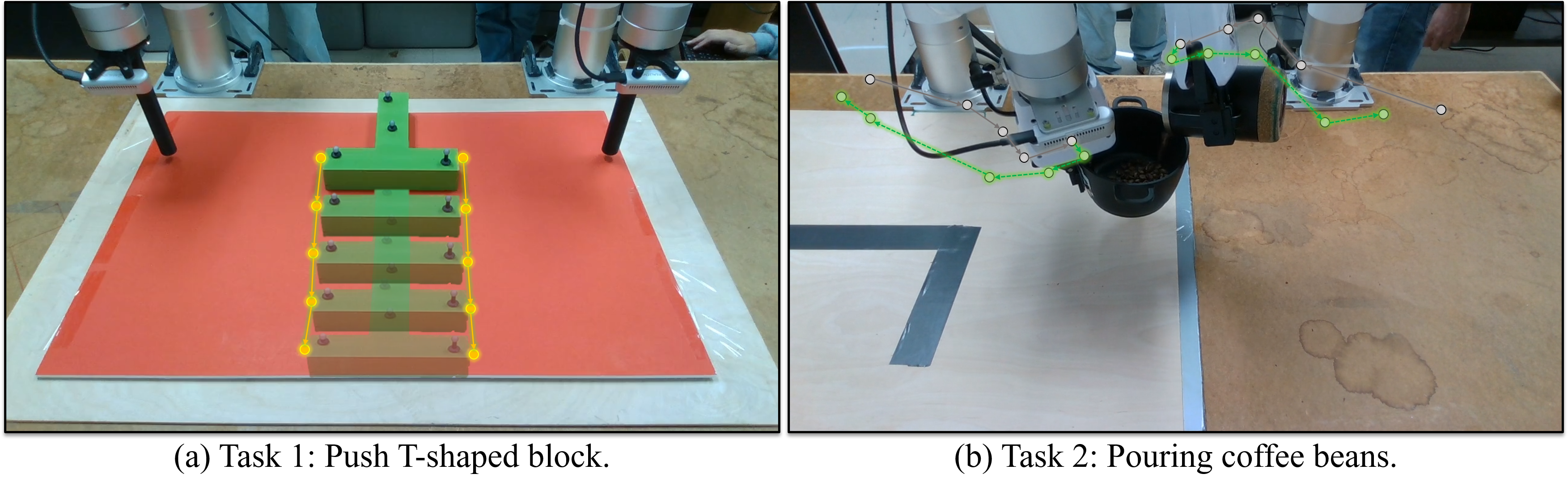}
\vspace{-0.3cm}
\caption{Cooperative Manipulation Tasks. In (a), the yellow dots and lines represent the ideal trajectory of the T-shaped block, with the expectation that its orientation remains largely unchanged throughout the motion. In (b), the gray dots and solid lines indicate the past trajectories of the two arms, while the green dots and dashed lines depict their future planned trajectories.}
\label{fig:tasks}
\vspace{-0.4cm}
\end{figure}

We define a cooperative task as one that requires the collaboration of both arms to be successfully completed. 
In such tasks, the actions of the two arms are interdependent, and the task cannot be accomplished through independent behavior from either arm alone. 
Based on this definition, we design two cooperative tasks, illustrated in Fig~\ref{fig:tasks}.
Task 1: Push-T. 
Unlike the traditional Push-T task, which only requires pushing the T-shaped block to a target location, our version imposes an additional constraint: the block is expected to maintain its initial orientation throughout the movement. 
Specifically, we require the T block to remain parallel to its initial posture during the entire trajectory, including at the final target position. 
This demands precise coordination between the two arms to apply force symmetrically and prevent rotation.
Task 2: Pouring coffee beans. 
In this task, the two arms start from fixed positions, collaboratively pour coffee beans from a cup into a small pot, and then return to a safe resting position. 
Since the task does not have a fixed target pose for the arms, success depends on the arms correctly interpreting each other’s intentions in real time. 
A misalignment in timing or trajectory can result in spillage, making the task a clear example of action interdependence and requiring tightly coupled coordination.

\subsection{Results}

\begin{wrapfigure}{r}{0.43\textwidth} 
  \vspace{-1.7cm}
  \vspace{-0.5\baselineskip}         
  \centering
  \setlength{\abovecaptionskip}{2pt}
  \setlength{\belowcaptionskip}{2pt}
  \begin{minipage}[t]{0.32\linewidth}
    \centering
    \includegraphics[width=\linewidth]{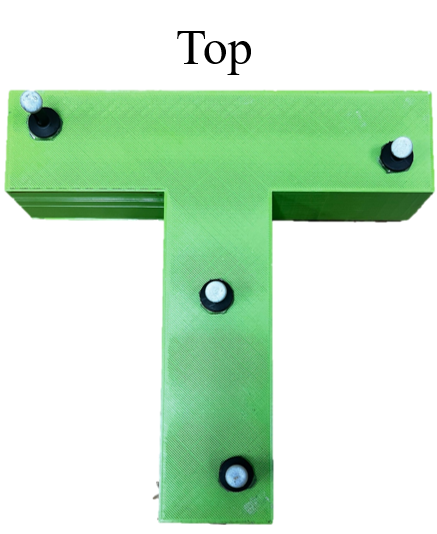}
    \subcaption{}
     \label{fig:tblock_top}  
  \end{minipage}%
  \hfill%
  \begin{minipage}[t]{0.32\linewidth}
    \centering
    \includegraphics[width=\linewidth]{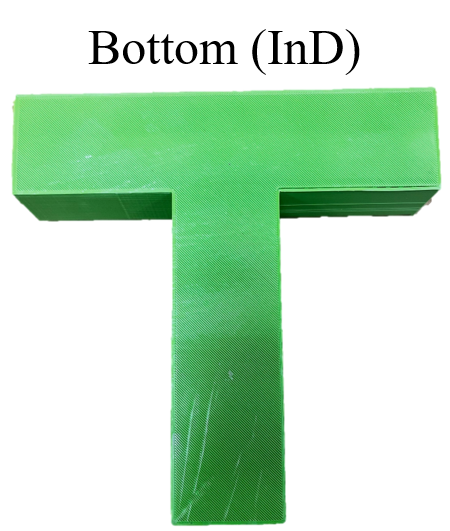}
    \subcaption{}
    \label{fig:tblock_botInD}
  \end{minipage}%
  \hfill%
  \begin{minipage}[t]{0.32\linewidth}
    \centering
    \includegraphics[width=\linewidth]{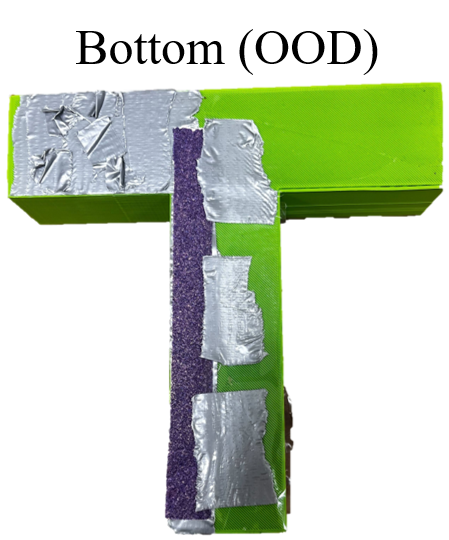}
    \subcaption{}
    \label{fig:tblock_botOOD}
  \end{minipage}
  \caption{T block in the fully InD (\ref{fig:tblock_top},\ref{fig:tblock_botInD}) and partially OOD (asymmetric friction coefficients, \ref{fig:tblock_botOOD}) cases.}
  \label{fig:task1_setup}
  \vspace{-0.4cm}
\end{wrapfigure}
In this section, we conduct comparative experiments using five methods across Tasks 1 and 2: (1) Centralized Diffusion Policy (i.e., vanilla diffusion policy); (2) Naive Decentralized Diffusion Policy (where we only use its own end-effector and third-view cameras during training for each arm’s policy); (3) Naive Consensus based Decentralized Diffusion Policy (without structural separation of ego and consensus embeddings, and lacking ToM constraints and directional consensus); (4) LatentToM without communication; and (5) LatentToM with Sheaf Laplacian (to achieve more stable performance at the cost of one communication step).

\subsubsection{Task~1: Push T-shaped Block}

In Task 1, we explored each method's ability to collaborate when confronted with a T-block that is visually in distribution (InD) (see Fig~\ref{fig:tblock_top}) but has out-of-distribution (OOD) asymmetric dynamics (see Fig~\ref{fig:tblock_botOOD}). More specifically, the underside of the block was modified so that each side would have distinct coefficients of friction, thereby exacerbating any error due to poor coordination. The qualitative results for each approach are shown in Fig~\ref{fig:result_task1_}. As expected, the Centralized Diffusion Policy (CDP), with its complete information, is able to account for the mismatch in dynamics and complete the task near-perfectly. Whereas, the Naive Decentralized Diffusion Policy (NDDP) fails to complete the task with a large rotational error. This is likely due to the small positional errors normally observed during deployment being amplified by the different and unique coefficients of friction used on each side. Similarly, we observe that the Naive Consensus based Decentralized Diffusion Policy (NCDDP) also fails to complete the task. Although NCDDP is trained to produce a shared consensus embedding, it is not guaranteed that this embedding is informative. Therefore, we suspect that it was not enough to capture the change in dynamics. 

\begin{figure}[t]
\centering
\includegraphics[width=4.8in]{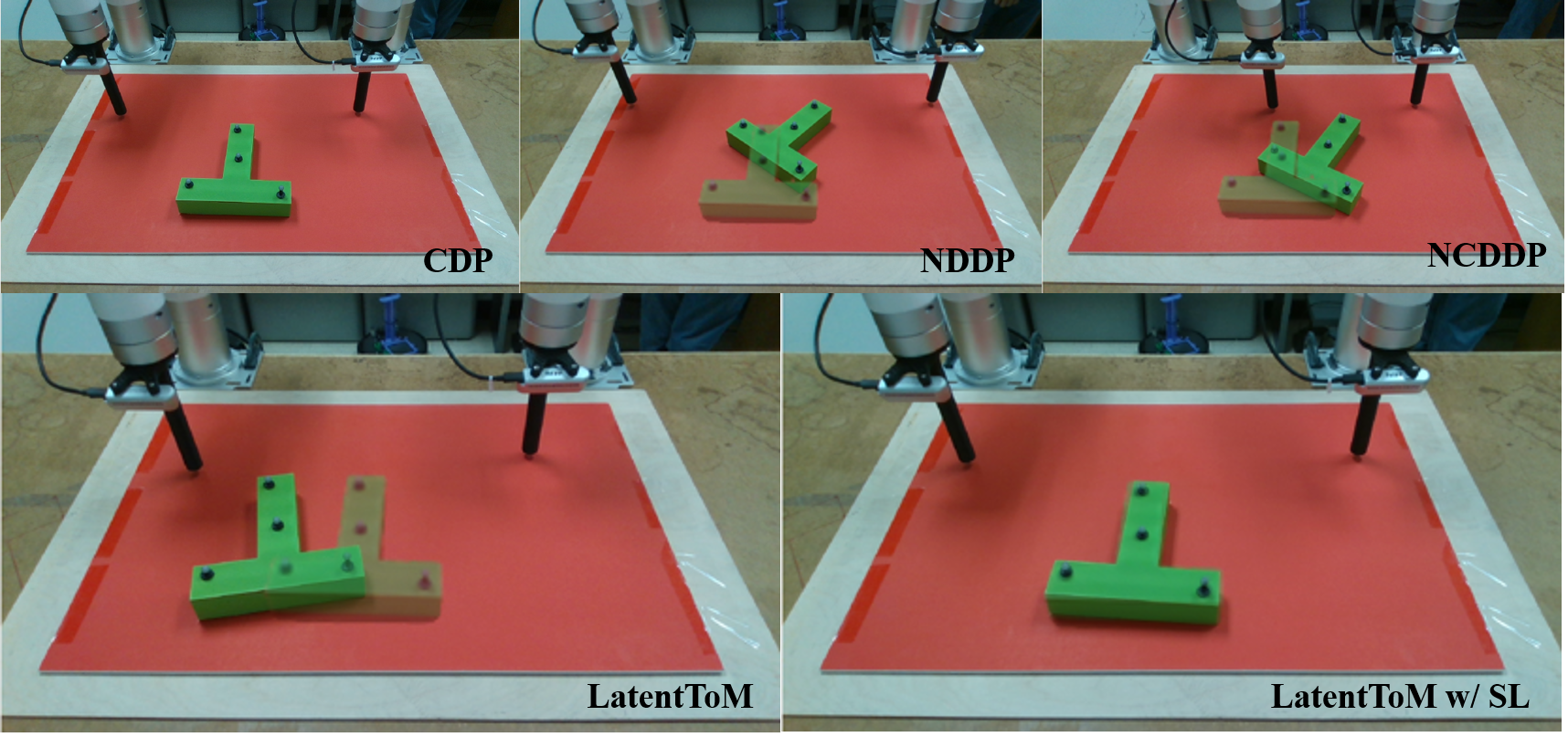}
\vspace{-0.2cm}
\caption{Rollout results of different methods under unbalanced friction setup as shown in Fig~\ref{fig:task1_setup}. The transparent image shows the final position pushed by CDP, which we treat as the expected position for comparing the outcomes of other methods.}
\label{fig:result_task1_}
\vspace{-0.4cm}
\end{figure}

Unlike the other two decentralized methods, our methods LatentToM and LatentToM with sheaf Laplacian are able to partially and fully complete the task respectively. We believe that our use of Theory of mind and a 
directional consensus mechanism ensures that the learned consensus embedding will be informative, therefore allowing the agents to react to environmental changes, while the sheaf Laplacian further helps mitigate the effect of noise, thereby allowing the agents to account for OOD changes and reach high levels of consensual collaboration.

\subsubsection{Task~2: Pouring Coffee Beans}

\begin{wraptable}{r}{0.4\textwidth}  
  \vspace{-0.9cm}
  \centering
  \vspace{-0.3cm}
  \caption{Pouring coffee beans results.}
  \vspace{0.1cm}
  \scalebox{0.7}{
    \begin{tabular}{c|ccccc}
    \toprule              
    \textbf{Methods} & \textbf{FS}   & \textbf{NR}   & \textbf{SO} & \textbf{AC}  & \textbf{CF}  \\
    \midrule\midrule                                 
    CDP        & 15  & 0   &  0  &  0  &  0 \\
    \midrule                              
    NDDP       &  7  & 0   &  0  &  0  &  8 \\
    \midrule                             
    NCDDP      &  9  & 0   &  1  &  1  &  4  \\
    \midrule 
    LatentToM & 13  & 2   &  0  &  0  &  0 \\
    \midrule 
    LatentToM w/ SL  & \textbf{14}  & \textbf{0}   &  \textbf{1}  &  \textbf{0}  &  \textbf{0}  \\
    \midrule 
    \bottomrule
\end{tabular}
  }
  \label{tab:result_task2}
  \vspace{-0.3cm}
\end{wraptable}
    
In addition to reporting the overall success rate, we provide a more fine-grained analysis of outcomes in the coffee bean pouring task. As shown in Table~\ref{tab:result_task2}, we classify each trial into one of five mutually exclusive outcome types:
(1) Fully Successful (FS): The cup is aligned correctly and all coffee beans are poured into the pot without spillage.
(2) Clear Failure (CF): A large amount of beans is spilled due to significant misalignment between the arms.
(3) No Return (NR): At least one arm fails to return to the resting position after pouring, indicating delayed or incomplete execution.
(4) Spill Out (SO): A minor spill occurs (typically in NR cases), but the main pouring is largely aligned.
(5) Arm Collision (AC): The two arms collide during the task, indicating poor spatial coordination.
We evaluated five methods over $15$ independent rollouts each. As expected, Centralized Diffusion Policy (CDP) is able to complete the pour task without any errors. Meanwhile, the Naive Decentralized Diffusion Policy (NDDP) performed the worst, with only $7$ fully successful cases and $8$ clear failures, demonstrating frequent miscoordination.
In contrast, the Naive Consensus based Decentralized Diffusion Policy (NCDDP) reduced failure cases by half, with only $4$ clear failures and $1$ collision, but still lacked robust recovery mechanisms.
We believe this result provides preliminary evidence that well-maintained consensus can lead to effective performance improvements.
Our method, LatentToM w/ Sheaf Laplacian (SL), achieved $14$ fully successful trials out of $15$, with only $1$ minor failure caused by delayed arm retraction due to bean dynamics rather than coordination error.
Even the LatentToM model without any communication performed reliably, with $13$ full successes and only $2$ cases of minor non-return behavior, both of which did not lead to any collisions or major spillage.
Notably, our method outperformed NDDP baseline by $46.7\%$ in full success rate, considering only fully successful (FS) cases as successes and treating all other outcomes as failures.
These results confirm that structured coordination mechanisms not only improve task success, but also eliminate subtle failure modes that naive methods cannot resolve.
\begin{figure}[t]
\centering
\includegraphics[width=5.4in]{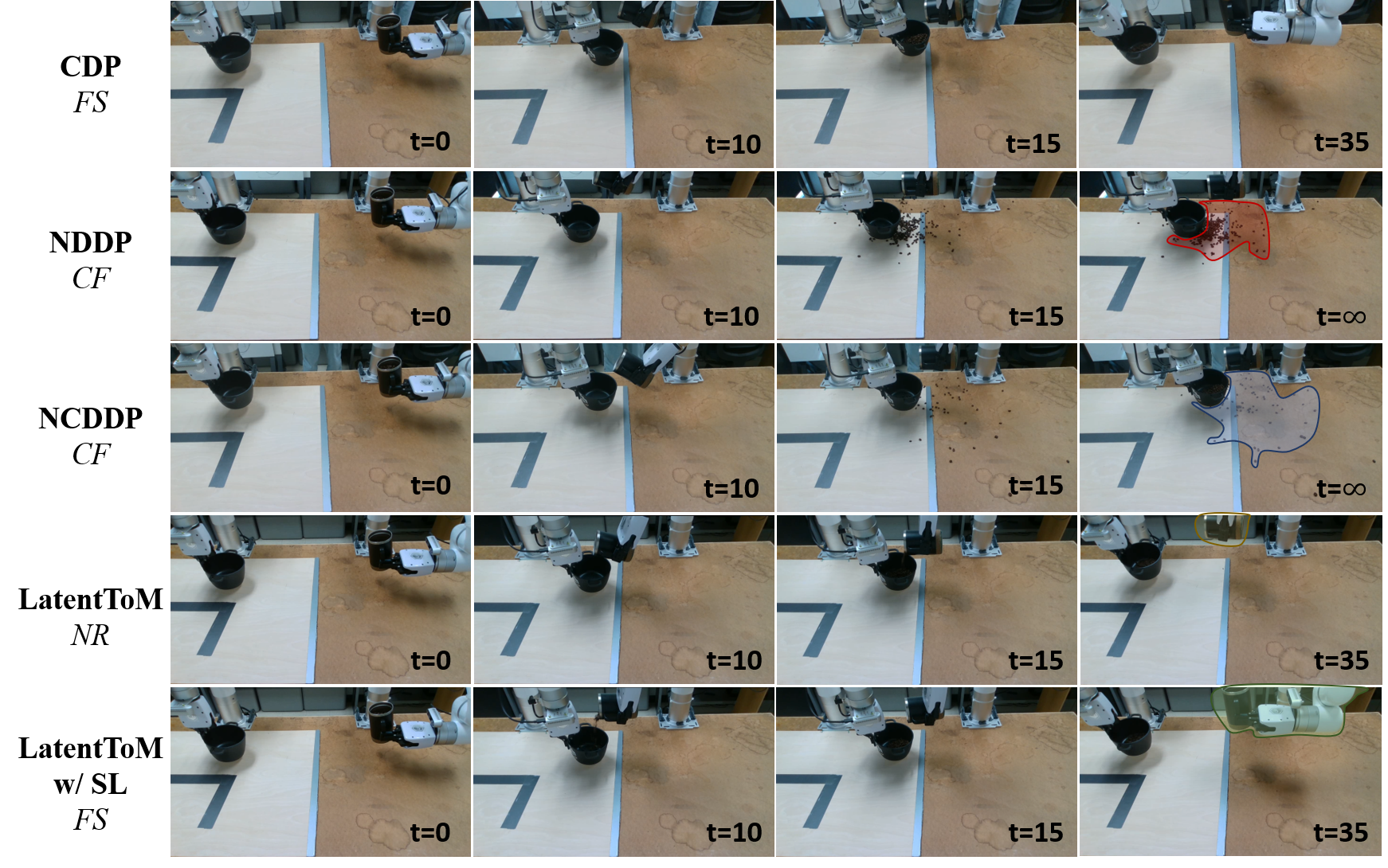}
\vspace{-0.3cm}
\caption{ Representative rollout results for Task 2 (coffee bean pouring). As shown in the figure, although both NNDP and NCDDP resulted in failures, NNDP clearly spilled more coffee beans than NCDDP (red area vs. blue area). For our method, Latent w/o SL exhibits an issue where, after completing the task, the arm may holding the cup fails to return to the resting position and remains in a risky posture, potentially causing additional spillage (highlighted in the yellow area). In contrast, DDP w/ SL successfully returns to a safe resting position, as indicated by the green area. However, for the same fully successful (FS) cases, CDP often completes the task more quickly.}
\label{fig:result_task2}
\vspace{-0.4cm}
\end{figure}


\section{Conclusion}
\label{sec:conclusion}
In this work, we propose a decentralized diffusion policy architecture called LatentToM to deal with multi-arm cooperation manipulation tasks.
Our key innovation lies in enabling each robot to explicitly maintain two distinct latent representations: an ego embedding, which encodes robot-specific information, and a consensus embedding, trained to capture shared, scene-level information across robots.
By leveraging the 1-cohomology from sheaf theory to guide consensus representation learning and introducing structural constraints inspired by Theory of Mind and a directional consensus mechanism, our method achieves both consistency and expressiveness in decentralized coordination.
Our experimental results demonstrate that LatentToM achieves competitive cooperative performance, matching state-of-the-art centralized baselines while significantly outperforming naive decentralized approaches. 
We further proposed to use the sheaf Laplacian as an optional online adjustment method to further enhance stability during inference without requiring model modifications.
We believe our approach offers an effective pathway for scaling diffusion-based control policies to larger and more complex multi-robot systems, paving the way for future research into robust and scalable decentralized collaboration frameworks.


\clearpage


\section{Limitations}

In this work, due to physical and computational hardware constraints, our current experiments were conducted with two robotic arms only.
While our approach is theoretically extensible to larger-scale multiarm systems, we have not yet evaluated its scalability beyond two agents in real-world settings.
Scaling up would likely require better model performance, but also pose significant challenges in data collection and system integration.
In future work, given adequate hardware support, we plan to design tasks involving highly homogeneous but larger-scale multiarm systems, such as object-passing scenarios among multiple arms.
These tasks can potentially be trained using data from smaller subsystems and then fine-tuned for deployment in larger setups, offering a scalable training-to-deployment pipeline.
Additionally, our current method relies on a fixed third-person camera to generate shared observations to construct consensus embeddings. 
In environments with occlusions or limited camera coverage, the quality of the consensus embedding may degrade, potentially impacting coordination.
Future work may also explore fusing multiple third-person views or developing view-invariant representations to improve robustness in such scenarios.


\bibliography{reference}  

\clearpage
\appendix

\section{Implementation Details}

In this section, we provide detailed implementation details about the ToM Predictor and Confidence Predictor introduced in Section \ref{subsec:mt}, where the main focus is the neural network architectures employed in our experiments. 
While these architectures are used in our current setup, they are not mandatory and can be reasonably adapted or improved in future research.

\subsection{ToM Predictor}
\label{appx:tom}

The Theory of Mind (ToM) predictor is a neural network $\psi_{u\rightarrow v}$  that aims to predict the ego embedding $h^{ego}_v$, representing the internal state/intention, of another arm $v$, based on the shared observation embedding $h_u^{con}$.

\begin{figure}[htp]
\centering
\includegraphics[width=5in]{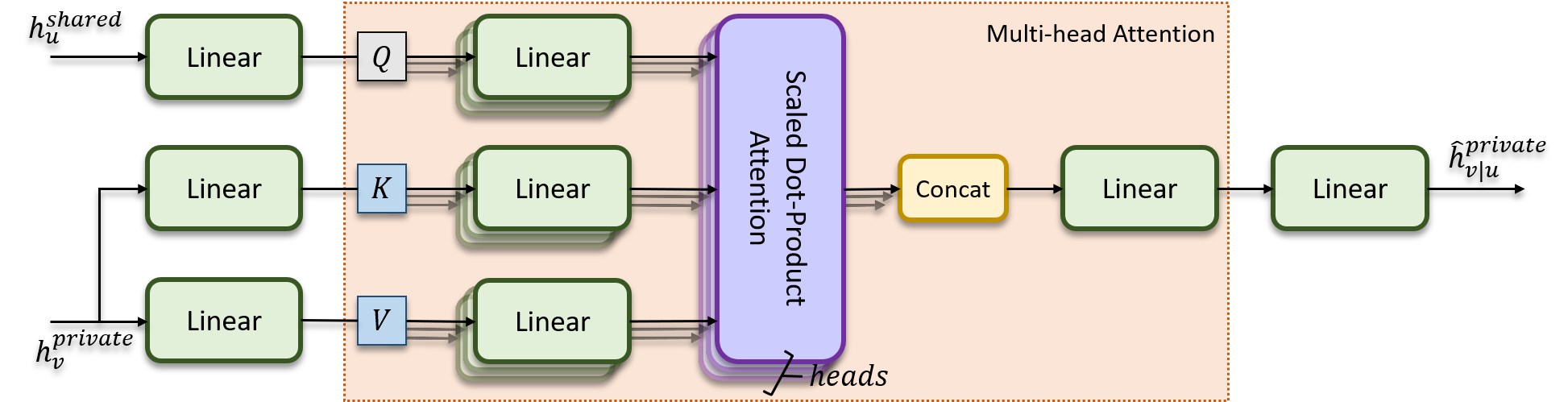}
\vspace{-0.1cm}
\caption{ToM Predictor.}
\label{fig:tom_predictor}
\vspace{-0.4cm}
\end{figure}

As illustrated in Fig~\ref{fig:tom_predictor}, we design the ToM predictor using a multi-head attention architecture. 
In this design, the consensus embedding of the ego arm $h_u^{con}$ serves as the query, while the ego embedding of $v$ acts as both the key and value. 
This design reflects the core idea of Theory of Mind: an agent (the query) actively reasons about another agent’s internal state by attending to observable cues (the key-value pairs). 
Moreover, this attention-based structure is inherently scalable and can be naturally extended to settings where the ego arm interacts with multiple neighboring arms.
Since the roles and tasks of each arm may differ significantly, we do not share parameters between ToM predictors. 
Instead, each direction $u\rightarrow v$ and $v\rightarrow u$ is modeled with its own independent network $\psi_{u\rightarrow v}$ and $\psi_{v\rightarrow u}$.

\subsection{Confidence Predictor}
\label{appx:conf}

In Section \ref{subsec:mt}, we introduce a directional consensus mechanism, in which each arm learns a confidence score $c_{u/v}\in [0,1]$ to assess the credibility of its embedding $h_{u/v}^{con}\in\mathbb{R}^{d_{con}}$ for alignment and collaboration within a multi-arm system.
To achieve this, we design a lightweight confidence predictor. 
This predictor follows a shared trunk with head-specific bias architecture, aimed at sharing the underlying environmental representation while allowing each arm to maintain independent confidence judgments.
The 
Specifically, the confidence module consists of a shared multi-layer feature extractor followed by agent-specific prediction heads:
\begin{equation}
    z_{u/v}=ReLU(W_2ReLU(LN(W_1h_{u/v}^{con}))), c_{u/v}=\sigma(W_{u/v}z_{u/v})
\end{equation}
where $W_1$, $W_2$ are the fully connected layers shared across arms, and $W_{u/v}$ is specific to each arm. 
$LN$ is the LayerNorm operator. $\sigma(\cdot)$ is the Sigmoid function.

\section{Proof}
\label{appx:proof}

Equation \ref{eq:sl} serves as an optional, classic bidirectional consistency synchronization operator provided during the model inference stage. 
For clarity and ease of reading in this section, we omit the superscript $con$ and change the number of iterations $t$ from subscript to superscript. 
For example, we use $h_u^t$ to denote $h_{u,t}^{con}$ from Equation \ref{eq:sl}. 
The clear form is given as follows:
\begin{equation}\label{eq:slc}
    \begin{aligned}
        \begin{bmatrix}
        h_u^{(t+1)} \\ h_v^{(t+1)}
        \end{bmatrix} = \underbrace{ 
        \begin{bmatrix}
        1 - \eta & \eta \\
        \eta & 1 - \eta
        \end{bmatrix}}_{\text{consistency operator}}
        \cdot
        \begin{bmatrix}
        h_u^{(t)} \\ h_v^{(t)}
        \end{bmatrix}.
    \end{aligned}
\end{equation}

\begin{theorem}
    Equation \ref{eq:slc} is equivalent to performing a low-order sheaf Laplacian step on the two embeddings.
\end{theorem}

\begin{proof}
    For an edge $e=(u,v)\in\mathcal{E}$, if these two nodes $u$ and $v$ have their own embedding $h_u$ and $h_v$, and have and restriction maps like:
    \begin{equation}
        \begin{aligned}
            \rho_{u\rightarrow e}:\mathcal{F}(u)\rightarrow\mathcal{F}(e),~\rho_{v\rightarrow e}:\mathcal{F}(v)\rightarrow\mathcal{F}(e)
        \end{aligned}
    \end{equation}
    According to sheaf theory, Then the consistency error on edge $e$ can be represented as:
    \begin{equation}
        \begin{aligned}
            e_{uv}=\rho_{u\rightarrow e}(h_u)-\rho_{v\rightarrow e}(h_v)
        \end{aligned}
    \end{equation}
    The aggregate update of the Sheaf Laplacian for node $u$ is: 
    \begin{equation}\label{eq:delta}
        \begin{aligned}
            \Delta_{\mathcal{F}}h_u=\sum_{v\in\mathcal{N}(u)}\rho^{-1}_{u\rightarrow e}(\rho_{u\rightarrow e}(h_u)-\rho_{v\rightarrow e}(h_v))
        \end{aligned}
    \end{equation} 
    which is the combination of structural consistency deviations.
    When the restriction map is identity map (i.e. $\rho_{u/v\rightarrow e}=I$), Equation \ref{eq:delta} can be simplified as:
    \begin{equation}
        \begin{aligned}
            \Delta_{\mathcal{F}}h_u=\sum_{v\in\mathcal{N}(u)}(h_u-h_v),
        \end{aligned}
    \end{equation}
    which is one of the most common low-order form of the combinatorial graph Laplacian applied to the sheaf.
    For the dual-arm system with two nodes in the graph, the above equation can be written as:
    \begin{equation}
        \begin{aligned}
            \Delta_{\mathcal{F}}h_u=h_u-h_v,~\Delta_{\mathcal{F}}h_v=h_v-h_u=-\Delta_{\mathcal{F}}h_u
        \end{aligned}
    \end{equation}
    The Sheaf Laplacian then becomes:
    \begin{equation}
        \begin{aligned}
            h_u^{(t+1)}&=h_u^{(t)}-\eta\cdot\Delta_{\mathcal{F}}h_u=h_u^{(t)}-\eta(h_u^{(t)}-h_v^{(t)})\\
            h_v^{(t+1)}&=h_v^{(t)}-\eta\cdot\Delta_{\mathcal{F}}h_v=h_u^{(t)}-\eta(h_v^{(t)}-h_u^{(t)})
        \end{aligned}
    \end{equation}
    By reorganizing the above equations into matrix form, we obtain the classic bidirectional consistency synchronization operator, as shown in Equations \ref{eq:sl} and \ref{eq:slc}.
\end{proof}

According to the above proof, we can claim that the bidirectional consistency synchronization operator is a special case of sheaf Laplacian operator.

\end{document}